\documentclass[runningheads]{llncs}
\usepackage{times}
\usepackage{helvet}
\usepackage{courier}
\usepackage[hyphens]{url}
\usepackage{graphicx}
\usepackage{amsfonts}
\usepackage{amsmath}
\usepackage{booktabs}
\usepackage{algorithm}
\usepackage{multirow}
\usepackage{hyperref}
\usepackage{adjustbox}
\usepackage{caption} 
\usepackage{multirow}
\usepackage{appendix}
\usepackage[table,xcdraw]{xcolor}
\DeclareMathOperator*{\aff}{aff}
\usepackage{color}
\usepackage{setspace}
\usepackage[noend]{algpseudocode}

\newcommand{\comm}[1]{\textcolor{blue}{#1}}
\newcommand{\red}[1]{\textcolor{red}{#1}}

\begin{document}
\title{Reachability Analysis  for
Feed-Forward Neural Networks using Face Lattices}
%
%
\author{Xiaodong Yang\inst{1} \and Hoang-Dung Tran\inst{1} \and Weiming Xiang\inst{2} \and Taylor Johnson\inst{1} }
\institute{}
\institute{Vanderbilt University, Nashville, TN 37212, USA \\
\email{\{xiaodong.yang,dung.h.tran, taylor.johnson\}@vanderbilt.edu}\\
\and Augusta University, Augusta, GA 30912 \\ \email{xiangwming@gmail.com}}

\maketitle              
\begin{abstract}
Deep neural networks have been widely applied as an effective approach to handle complex and practical problems. However, one of the most fundamental open problems is the lack of formal methods to analyze the safety of their behaviors. To address this challenge, we propose a parallelizable technique to compute exact reachable sets of a neural network to an input set. Our method currently focuses on feed-forward neural networks with ReLU activation functions. One of the primary challenges for polytope-based approaches is identifying the intersection between intermediate polytopes and hyperplanes from neurons. In this regard, we present a new approach to construct the polytopes with the face lattice, a complete combinatorial structure. The correctness and performance of our methodology are evaluated by verifying the safety of ACAS Xu networks and other benchmarks. Compared to state-of-the-art methods such as Reluplex, Marabou, and NNV, our approach exhibits a significantly higher efficiency. Additionally, our approach is capable of constructing the complete input set given an output set, so that any input that leads to safety violation can be tracked.

\end{abstract}
\section{Introduction}
Deep neural networks (DNNs) have been playing a critical role in handling complex and practical problems and are being widely applied in safety-critical, autonomous cyber-physical systems (CPS). However, one major obstacle in realizing autonomous CPS is the lack of formal methods to provide guarantees on functional behaviors, such as in safety-critical systems like autonomous motor vehicles and robotic surgery machines. A slight perturbation in the input of a neural network may lead to an error behavior in output \cite{moosavi2016deepfool}. Recently, there has been significant effort to develop methods to establish robustness and formal guarantees of learning-enabled components (LECs) like DNNs~\cite{pulina2010abstraction,huang2017safety,gehr2018ai2,dutta2017output,katz2017reluplex,xiang2017reachable,xiang2018output,wang2018formal,zhang2018efficient,singh2018fast,singh2019abstract,wang2018efficient,tran2019parallelizable,tran2019fm}. There are two primary classes of reachability methodologies to formally analyze DNNs, over-approximation and exact (or complete) analysis. The over-approximation methods are mainly based on mixed-integer linear programs (MILPs)~\cite{lomuscio2017approach,dutta2017output,kouvaros2018formal}, zonotopes~\cite{gehr2018ai2}, abstract domain ~\cite{singh2019abstract}, and linearization~\cite{weng2018towards,zhang2018efficient}. These methods can only guarantee the soundness of the analysis but not completeness. The MILPs-based work ~\cite{dutta2017output} can guarantee both aspects when neural networks have only one output, but fails when they have multiple outputs. Because it estimates the range of outputs independently and eventually generates a box domain that over approximates exact reachable sets. Most of the over-approximation methods are capable of efficiently analyzing large scales neural networks based ReLU activation functions. Their main strategy is replacing each ReLU function with a more conservative domain so that the number of output reachable sets of each layer can be largely reduced. But these approaches are only limited to point-wise inputs with a very little perturbation, and their conservativeness of the estimated reachable sets will exponentially grow as the input domain increases. This is already studied in \cite{tran2019fm}.


While the exact analysis is mainly based on the satisfiability modulo theory (SMT)~\cite{katz2017reluplex,katz2019marabou}, polytopes~\cite{xiang2017reachable} and Star set in a tool named NNV~\cite{tran2019fm}. Through these approaches, both soundness and completeness of the verification can be guaranteed. The strategy of Reluplex~\cite{katz2017reluplex} is extending the simplex method to handle the piece-wise linear ReLU activation function by allowing variables of defined ReLU pairs to temporarily violate their semantics. The Marabou~\cite{katz2019marabou} is an improved version of Reluplex. It supports arbitrary piecewise-linear activation function, parallel computation, etc. Another work ReluVal~\cite{wang2018formal} is based on interval arithmetic to over approximate the bounds on the outputs. The influence of each input variable on the output is analyzed so that it
can repeatedly split the input intervals and efficiently refine the output range. All these works focus on the satisfiability problem that is to determine whether there exists output that locates in the unsafe domain. While the works~\cite{xiang2017reachable,tran2019fm} are conducting the reachability analysis of a neural network. Given an input set, they can compute the exact output sets of a neural network. This reachability analysis is also associated with the quantification of linear regions of neural networks~\cite{montufar2014number,serra2017bounding,hanin2019complexity}. A linear region of a piecewise linear functions $F: \mathbb{R}^n\rightarrow \mathbb{R}^m$ refers to a maximum convex subset of an input set in $\mathbb{R}^n$, on which the function $F$ is linear. Accordingly, the input set is splitted by the ReLU function in each neuron into pieces which are linear regions. Each output set of a network computed by \cite{xiang2017reachable,tran2019fm} corresponds to the output with respect to a linear region. Therefore, the number of output sets computed is equal to the number of linear regions. These methods can provide a full understanding of the neural network's behavior and are promising directions towards safe networks.

However, improvement of the efficiency is challenging. The work~\cite{xiang2017reachable} is constructing the input set with a polytope. It has three basic operations when a polytope passes through one layer of a neural network. They are respectively, \textit{affine transformation} by the weight matrix and bias vector between layers, \textit{intersection and division} by hyperplanes from the non-differentiable point at $0$ in ReLU function of each neuron, and \textit{projection} on the hyperplane due to the property of the negative domain in ReLU function. One of the primary challenges is to identify the intersection of a polytope with a hyperplane. To solve this problem, they use $H$-representation ($H$-rep) which is a set of linear inequalities to check the feasibility with another linear inequality from the hyperplane. This computation is treated as a linear programming problem. For the \textit{affine transformation} and \textit{projection}, the $H$-rep is transformed to $V$-representation ($V$-rep) which is a set of vertices. However, the change between $H$-rep and $V$-rep are respectively a vertex enumeration problem and a facet enumeration problem in computational geometry, both of which have high computational complexity. Therefore, the constant representation switching and feasibility inspecting lead to undesirable efficiency. The work~\cite{tran2019fm} proposed the Star set which can avoid the transformation between $H$-rep and the $V$-rep. Instead, it identifies the intersection using optimization. But a large number of neurons will yield too many optimization processes and therefore impact its efficiency.

As a complementary approach, we present a novel face-lattice-based and parallelizable methodology that is capable of the sound and complete reachability analysis with higher efficiency. This method currently focuses on feed-forward neural networks with ReLU activation. To overcome the challenge, we develop an intuitive approach that is to inspect the distribution of a polytope's vertices on the sides of the hyperplane. If vertices scatter on both sides of or on the hyperplane, the intersection happens. Otherwise, it doesn't. The following problem is how to identify polytopes and their vertices generated from this intersection so that their complete structure information can be maintained for future operations. Here, we introduce the face lattice to encode the complete combinatorial structure for polytopes and preserves the adjacency between faces of the polytope, so that those three types of operations on polytopes can be more efficiently conducted.

\section{Preliminaries}

\subsection{Feedforward Neural Networks}
A FNN consists of one input layer, multiple hidden layers, and one output layer. Each layer contains multiple neurons which are interconnected with neurons in the next layer by weights in a feed-forward way. The output of each neuron is associated with three components: its input weight $\omega$, input bias $b$ and the activation function $f$, as shown in:
\begin{equation*}
    y_i = f(\sum\nolimits_{j=1}^{n}\omega_{i,j}x_{j} + b_j)
\end{equation*}
where $\omega_{ij}$ and $b_j$ are respectively the weight and bias from the $j$th neuron of the previous layer to the $i$th neuron of the current layer, and $x_{j}$ is an input to this neuron and also the output of $j$th neuron in the front layer, and $y_{i}$ is the output of the $i$th neuron. In this paper, we consider the ReLU activation which is defined as $ReLU(x) = \max(0,x)$. Let $W_{(k,k\text{-}1)}$, $b_k$ denote the weight matrix, the bias vector between the $(k\text{-}1)$th layer and $k$th layer, and $X_{k\text{-}1}$ be its input consisting of elements $x_j$, then the output of the $k$th layer will be
\begin{equation}
    \Phi_k(X_{k}) =  ReLU(W_{(k,k\text{-}1)}X_{k} + b_k)
    \label{equ:layer}
\end{equation}
For the first hidden layer, its input $X_{0}$ is also the input to the network. Besides, the output of one layer is also an input of the next layer. Therefore, given an input $X_{0}$, the output $Y_{k}$ of the $k$th hidden layer will be
\begin{equation}
    Y_{k} = \Phi_{k}(\Phi_{k-1}(\dots (\Phi_{1}(X_{0}))))
\end{equation}

\subsection{Face Lattices}
The geometric background of the face lattice structure is introduced in this section. It includes the concepts of supporting hyperplane, faces of polytopes as well as the properties of this structure. More other geometric details can be found in work~\cite{henk200416}.

\begin{definition}[Supporting Hyperplane]
    A hyperplane $\mathcal{H}$ denoted by $\textbf{a}^{\top}\textbf{x}=b$ is supporting polytope $P$ if one of its closed halfspaces, $\textbf{a}^{\top}\textbf{x}\le b$ or $\textbf{a}^{\top}\textbf{x}\ge b$ contains $P$.
    \label{def:sh}
\end{definition}

\begin{definition}[Face]
    The face of a convex polytope is an intersection of this polytope with a supporting hyperplane. When the polytope is full dimensional, the face is named a \textit{nontrivial face}. The polytope itself and the empty set are also its faces which are called \textit{trivial faces}. When the dimension of $\aff(P\cap \mathcal{H})$ is $k$, the face is denoted as $k\textbf{-}f$ or $k\textbf{-}face$. The function $\aff(S)$ indicates the affine hull of S, which is the smallest affine set that contains S.
    \label{def:face}
\end{definition}

A d-dimensional convex polytope $P$ is called full dimensional if it locates in $\mathbb{R}^{n=d}$, otherwise it is not full dimensional if it is in $\mathbb{R}^{n>d}$. The convex polytope $P$ consists of different dimensional daces. Regardless of $P$'s dimension and the space it locates in, it contains a set of $0\textbf{-}$faces, $1\textbf{-}$faces, $\dots$, $(d\text{-}2)\textbf{-}$faces, $(d\text{-}1)\textbf{-}$ faces as well as the empty set and itself. The $0\textbf{-}$face, $1\textbf{-}$face, $(d\text{-}2)\textbf{-}$face and $(d\text{-}1)\textbf{-}$ face are respectively named $vertex$, $edges$, $ridges$ and $facets$. 

The face lattice $\mathcal{L}(P)$ is a complete combinatorial structure that contains all faces of $P$ and partially orders them by face containment. One face $\mathcal{H}_i$ is said to contain a face $\mathcal{H}_j$ if the $\mathcal{H}_j$ is a subset of the $\mathcal{H}_i$. A tetrahedron and its face lattice are shown in (a) and (b) of Figure \ref{fig:combine}. The empty face $\emptyset$ is not considered for a conciser presentation. Besides the $3\textbf{-}face$ which is the tetrahedron itself, there are 14 faces described by blue blocks and their dimension ranges from 0 to 2. The face containment only consider higher-dimensional faces containing adjacent lower-dimensional faces. For instance, the $1\textbf{-}f_1$ ($edge_{2\textbf{-}3}$) contains $0\textbf{-}f_2$ ($vertex_2$) and $0\textbf{-}f_3$ ($vertex_3$), and the $1\textbf{-}f_1$ is contained in $2\textbf{-}f_1$ ($plane_{2\textbf{-}3\textbf{-}4}$) and $2\textbf{-}f_2$ ($plane_{1\textbf{-}2\textbf{-}3}$). The containment relation is expressed in a grey connection line. All the polytopes refer to convex polytopes that are expressed in the face lattice structure in the following sections

\section{Operations of Polytopes}
As discussed in the introduction, there are three basic operations in processing polytopes in the neural network. In terms of the sequential order, they are respectively, the \textit{affine transformation} by weights, \textit{intersection and division} by a hyperplane $\mathcal{H}$, and the \textit{projection} on $\mathcal{H}$. The affine transformation is mapping the current polytope to another space. The combinatorial theory considers polytopes that differ only by a change of coordinates (an affine transformation) are equivalent~\cite{henk200416}. The main reason is that the \textit{affine transformation} only changes vertices but preserve polytopes' combinatorial structure. Therefore, its face lattice stays unchanged under this transformation. The \textit{projection} on a hyperplane is orthogonally mapping the polytope on this hyperplane. Similarly to the affine transformation, it only works on its vertices. The rest of this section only focuses on the\textit{ intersection and division}. 

\begin{figure*}[ht]
  \includegraphics[scale = 0.55]{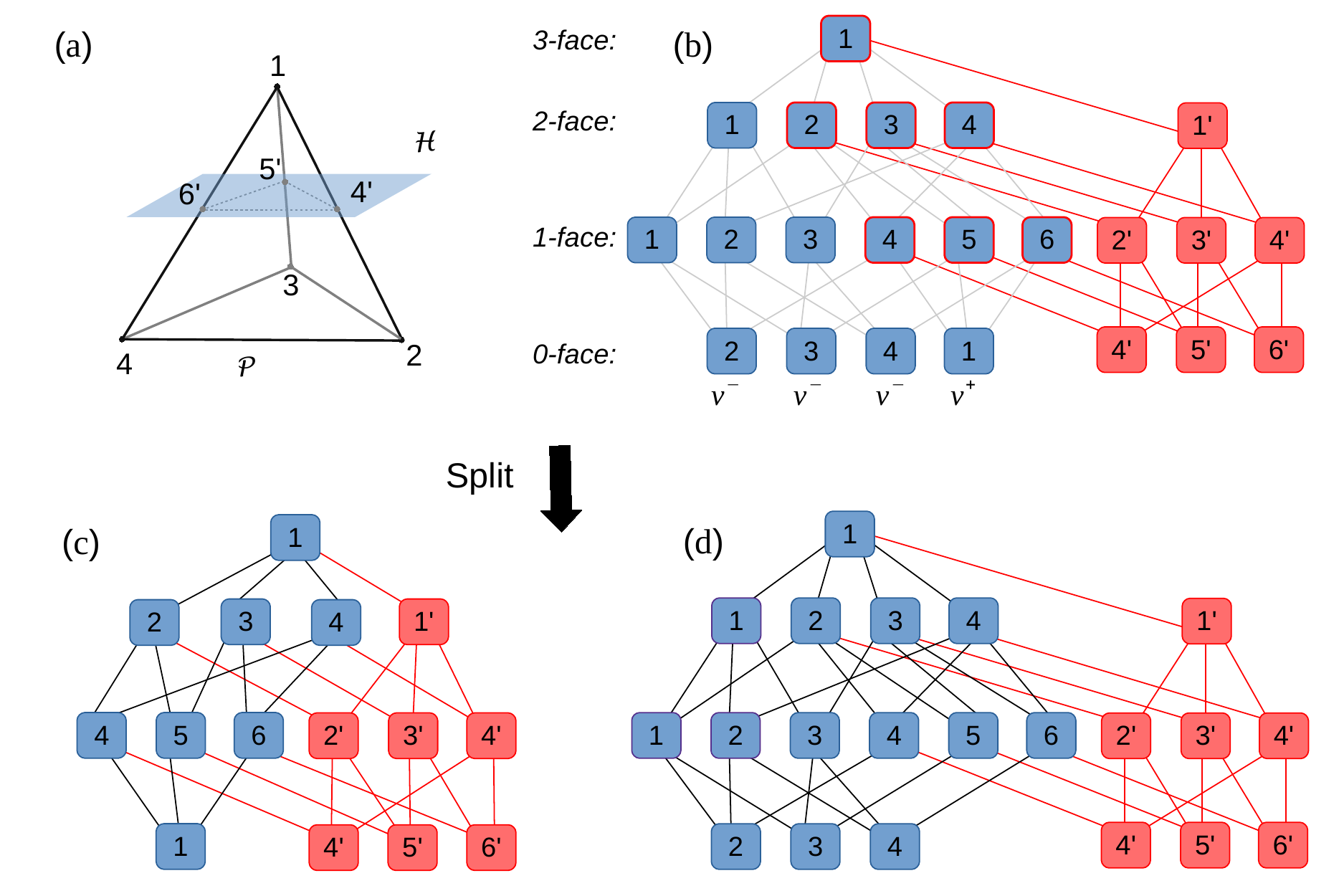}
  \centering
  \caption{A demonstration of the face lattice of a tetrahedron $P$ and its\textit{ intersection and division.} by a hyperplane $\mathcal{H}$. In (b), the blocks filled with blue and connected by grey lines is its face lattice which structure is based on the Hasse diagram. The number of each entry represents the face's dimension.  The grey connection only allows that high-dimensional faces contain low-dimensional faces. The blocks with red frames and blue color are the faces intersecting with $\mathcal{H}$ in $P$, which are extracted from $positive$ and $negative$ vertices in terms of the Lemma \ref{th1}. The face lattice of new generated faces is in red, which is extracted according to the Lemma \ref{th2}. The red connection between blue blocks and red blocks denote the containment relation between faces in $P$ and new faces. $(c)$ and $(d)$ represent the face lattice of the $positive$ and $negative$ polytopes, respectively.}
  \label{fig:combine}
\end{figure*}

\subsection{Intersection of Face Lattices with Hyperplanes}
\label{sec:inter}
Since the face lattice contains all the information of faces adjacency, determination of an intersection between a polytope $P$ and a hyperplane $\mathcal{H}$ can be simplified to finding the intersection of edges in $P$ with $\mathcal{H}$. It then can be realized by checking the distribution of edges' vertices on the side of $\mathcal{H}$. Let the $\mathcal{H}$ be described by $\textbf{a}^{\top}\textbf{x} + b = 0$. To compute the vertices' distribution, we can substitute the $\textbf{x}$ with vertices. Let vertices such that $\textbf{a}^{\top}\textbf{x}$ are greater than $0$, smaller than $0$, and equal to $0$ be respectively called $positive$ vertices $v^{+}$, $negative$ vertices $v^{-}$, and $zero$ vertices $v^{0}$. Accordingly, there are three cases in the \textit{intersection and division}:
\begin{enumerate}
    \item[(1)] There exist both $v^{+}$ and $v^{-}$. In this case, the hyperplane $\mathcal{H}$ intersects with $P$, and a set of new faces $S_f^*$ of dimensions between $0$ and $(d\text{-}1)$ are generated, where $f \subseteq \mathcal{H}$ and $f \subseteq P$ for all $f \in S_f^*$. The original polytope is thus divided into two non-empty sub-polytopes $P_{sub}^{+}$ which consists of nonnegative vertices, and $P_{sub}^{-}$ which consists of nonpositive vertices.
    \item[(2)] There are no $v^{-}$s. In such cases, polytope $P$ is contained in the closed halfspace $\textbf{a}^{\top}\textbf{x} \ge 0$. The sub-polytops $P_{sub}$ generated from this operation is the $P$ itself. It is a positive polytope $P_{sub}^{+}$ 
    \item[(3)] There are no $v^{+}$s. In such cases, polytope $P$ is contained in the closed halfspace $\textbf{a}^{\top}\textbf{x} \le 0$. The sub-polytops $P_{sub}$ generated from this operation is also the $P$ itself. It is a negative polytope $P_{sub}^{-}$
\end{enumerate}

Here we focus on the first case where the intersection occurs. To identify those two sub-polytopes, we first need to find the faces of $P$ that yield the $S_f^*$ in the intersection. Edges with both $positive$ and $negative$ vertices or at least having one $zero$ vertex are said to have intersection with the hyperplane $\mathcal{H}$. In terms of the Lemma \ref{th1}, higher-dimensional faces that also intersect with the $\mathcal{H}$ can be identified from these edges. Thus all faces of $P$ that intersect with $\mathcal{H}$ are obtained. Intersecting with $\mathcal{H}$, each of such face will generate a new face that is one-dimension lower. These new faces the $S_f^*$. The process is demonstrated in the (b) of Figure \ref{fig:combine}. For instance, the $2\textbf{-}f_2$ ($plane_{1\textbf{-}2\textbf{-}3}$) intersects with $\mathcal{H}$ and generates a new face $1\textbf{-}f_{2'}$ ($edge_{4'\textbf{-}5'}$) whose vertices $0\textbf{-}f_{4'}$ ($vertex_{4'}$) and $0\textbf{-}f_{5'}$ ($vertex_{5'}$) are respectively generated from the intersection of $\mathcal{H}$ with $1\textbf{-}f_4$ ($edge_{1\textbf{-}2}$) and $1\textbf{-}f_5$ ($edge_{1\textbf{-}3}$) .

Lemma \ref{th2} states that the face containment relation of the new faces $S_f^*$ is inherited from faces' in the polytope. An example is shown in (b) of Figure \ref{fig:combine}. The red blocks represent the new faces set $S_f^*$, and the red connection between the red blocks represent the inherited face containment relation. For instance, the relation that the new face $1\textbf{-}f_{3'}$ ($edge_{5'\textbf{-}6'}$) contains the new face $0\textbf{-}f_{6'}$ ($vertex_{6'}$) is derived from the relation that $2\textbf{-}f_{3}$ ($plane_{1\textbf{-}3\textbf{-}4}$) contains $1\textbf{-}f_{6}$ ($edge_{1\textbf{-}4}$) in $P$. Overall, these two steps in $intersection$ generate a new face lattice that includes all the faces of $P$ and all the new generated faces $S_f^*$. Then as shown in (c) and (d), the \textit{division} is splitting this face lattice into the $P_{sub}^{+}$ and $P_{sub}^{-}$ according to vertices.

\begin{lemma}
    \label{th1}
    If a $k\textbf{-}f$(k-dimensional face) intersects with a hyperplane $\mathcal{H}$, all the $(k\text{+}1)\textbf{-}f$s where $k\textbf{-}f \subseteq (k\text{+}1)\textbf{-}f$ also intersect with $\mathcal{H}$.
\end{lemma}
\begin{proof}
    $k\textbf{-}f \subseteq (k\text{+}1)\textbf{-}f$ indicates $k\textbf{-}f\cap \mathcal{H} \subseteq (k\text{+}1)\textbf{-}f\cap \mathcal{H}$. Since $k\textbf{-}f\cap \mathcal{H} \neq \emptyset$, then $(k\text{+}1)\textbf{-}f\cap \mathcal{H} \neq \emptyset$. Therefore, $(k\text{+}1)\textbf{-}f$ also intersect with $\mathcal{H}$.
\end{proof}

\begin{lemma}
    \label{th2}
    Given two faces $k\textbf{-}f$ and $(k\text{+}1)\textbf{-}f$ where $k\textbf{-}f \subseteq (k\text{+}1)\textbf{-}f$, and that they both intersect with a hyperplane $\mathcal{H}$ where $k\textbf{-}f \cap \mathcal{H} = (k\text{-}1)\textbf{-}f^{'}$ and $(k\text{+}1)\textbf{-}f \cap \mathcal{H} = k\textbf{-}f^{'}$, then $(k\text{-}1)\textbf{-}f^{'} \subseteq k\textbf{-}f^{'}$.
\end{lemma}
\begin{proof}
    $k\textbf{-}f \subseteq (k\text{+}1)\textbf{-}f$ indicates $k\textbf{-}f \cap \mathcal{H} \subseteq (k\text{+}1)\textbf{-}f \cap \mathcal{H}$. Therefore, $(k\text{-}1)\textbf{-}f^{'} \subseteq k\textbf{-}f^{'}$ can be derived.
\end{proof}

\section{Reachability Analysis with Face Lattices}

In this section, we use $\mathcal{I}_{k}$ and $\mathcal{O}_k$ to represent the input and output polytope sets of the $k$th layer. By substituting the $X_k$ with $\mathcal{I}_k$, and the output with $\mathcal{O}_k$ in the Equation \ref{equ:layer}, we obtain
\begin{equation*}
    \mathcal{O}_k =  ReLU(W_{(k,k\text{-}1)}\mathcal{I}_{k} + b_k)
\end{equation*} which indicates the three basic operations on polytopes. The operation on $\mathcal{I}_k$ inside the parentheses is the \textit{affine transformation}. The function $ReLU(*)$ indicates a sequence of\textit{ intersection and division}, and \textit{projection} on the transformed polytope. As claimed in the abstract, our approach can derive the complete input set given an output set. It is realized by incorporating the transformation tuple in Definition \ref{def:tuple} to maintain mathematical relation between an intermediate polytope $P_k$ and a subset of the initial input $P_0$ to the neural network. Thus, any violation of safe properties in output can be mapped back into the input. From the perspective of the \textit{linear region} presented in Introduction section, the Equation \ref{equ:trans} is one linear function and $P_{0}^{sub}$ is one linear region. 

\begin{definition}[Transformation Tuple]
    Given an input polytope $P_0$ to the neural network and a polytope $P_k \in \mathcal{O}_k$, a transformation tuple is denoted by $P_k \text{=}\langle P_{0}^{sub}, M_k, d_k \rangle$ where $P_{0}^{sub} \subseteq P_0$, such that 
    \begin{equation}
        P_k = M_kP_{0}^{sub} + d_k 
        \label{equ:trans}
    \end{equation}
    \label{def:tuple}
\end{definition}

\subsection{Processing of Transformation Tuple}
Take the output computation of the first layer for example, suppose we have an initial input $P_0$, then its transformation tuple can be initialized as $\langle P_0^{sub}, M_0, d_0\rangle$ where $P_0^{sub} = P_0$, $M_0$ is an identity matrix, and $d_0$ is an array of zeros. After the \textit{affine transformation } by weights $W_{(1,0)}$ and $b_{(1,0)}$, the $P_0^{sub}$ stays unchanged and the tuple will be updated to $\langle P_0^{sub}, M_1, d_1\rangle$, where 
\begin{equation}
     M_{1} = W_{(1,0)}M_0, ~~~ d_{1} = W_{(1,0)}d_0 + b_{(1,0)}
     \label{equ:linearTransform}. 
\end{equation} 

Let $P_1$ be this updated tuple. Afterwards, the \textit{intersection and division} operation will be conducted on this new tuple for each neuron in this layer. Suppose we start by processing the first neuron, we can first derive a hyperplane $\mathcal{H}_1: \textbf{a}^{\top}_1\textbf{x}+c_1 = 0$ where $\textbf{a}_1$ is an array with its first element $\textbf{a}[1]=1$ and the rest being zeros, and $c_1\text{=}0$. When considering the $i$the neuron, the $\textbf{a}[i]=1$ and the rest keeps the same. According to the ReLU function, the subset of $P_1$ that is in the closed halfspace $\textbf{a}^{\top}_1\textbf{x}+c_1 \ge 0$ stays unchanged. For the subset that is in the closed halfspace $\textbf{a}^{\top}_1\textbf{x}+c_1 \le 0$, all its $\textbf{x}[1]$ is set to 0. These subsets can be identified by the $intersection$ operation. Instead of determining the intersection of $P_1$ with $\mathcal{H}_1$, we choose to check the intersection of its $P_0^{sub}$ with a mapped-back hyperplane $\mathcal{H}_0$ from $\mathcal{H}_1$  by Equation \ref{equ:vertice_multiply} and \ref{equ:mappedh}. This approach is directly operating the initial input subset $P_0^{sub}$, which is helpful for the further computation. Besides, the \textit{intersection and division} on the $P_0^{sub}$ is also the process of splitting and generating linear regions.

The method of mapping $\mathcal{H}_1$ into the space of $P_0^{sub}$ and obtaining $\mathcal{H}_0$ is as following. From the Equation \ref{def:tuple}, we can derive that 
\begin{equation}
        \textbf{x}_1 = M_1\textbf{x}_0^{sub} + d_1
        \label{equ:vertice_multiply}
\end{equation}
where $\textbf{x}_0^{sub} \in P_0^{sub}$  and $\textbf{x}_1 \in  P_1$. By combining this equation with $\mathcal{H}_1$, we have
\begin{align}
    \textbf{a}_1^{\top} \textbf{x}_1 + c_1 & = \textbf{a}_1^{\top}(M_1\textbf{x}_0^{sub} + d_1) +c_1 \nonumber \\
    & = (\textbf{a}_{1}^{\top}M_1)\textbf{x}_0^{sub} + (\textbf{a}_1^{\top}d_1 + c_1) = 0
    \label{equ:mappedh}
\end{align}
Thus, we can obtain the mapped hyperplane $\mathcal{H}_0: \textbf{a}^{\top}_0\textbf{x}_0^{sub}+c_0 = 0$, where $\textbf{a}^{\top}_0 = \textbf{a}_1^{\top}M_1$ and $c_0 =  \textbf{a}_1^{\top}d_1 + c_1$. 

During the \textit{intersection and division} operation, the processing of the tuple depends on different cases. In the first case, $P_0^{sub}$ will be splitted into two polytopes $P_{0}^{sub+}$ and $P_{0}^{sub-}$ by the hyperplane $\mathcal{H}_0$. Correspondingly, the original tuple of $P_1$ will also be splitted and we achieve
\begin{equation*}
    P_1^+ = \langle P_{0}^{sub+}, M_1, d_1 \rangle,\ P_1^- = \langle P_{0}^{sub-}, M_1, d_1 \rangle
\end{equation*}
According to the property of the ReLU function, $P_1^+$ stays unchanged and no \textit{projection} operation is needed. While the points of the $P_1^-$ should be mapped to the $\mathcal{H}_1$ where the first element $\textbf{x}[1]$ of them will be changed to zero. As shown in the following equation, this operation can be realized by left multiplying $I_1$ with $P_1^-$ where $I_{1}$ is an identify matrix with the first diagonal entry being zero. When processing with the $i$th neuron, $I_i$ will be an identify matrix with $i$th diagonal entry being zero.
\begin{align}
    P_1^{'-} & = I_1P_1^{-} = I_1(M_1P_0^{sub-} + d_1) \nonumber \\
             & = (I_1 M_1) P_0^{sub-} + (I_1 d_1)
\label{equ:projection}
\end{align}
After the \textit{projection}, the $P_1^{'-} = \langle P_{0}^{sub-}, M_1', d_1' \rangle$, where
\begin{equation}
    M_1' = I_1 M_1, \ d_1' = I_1 d_1. 
    \label{equ: md_update}
\end{equation}
While in the second case, only $P_{0}^{sub+}$ is generated from the \textit{intersection and division} on $P_{0}^{sub}$. Then $P_{0}^{sub+} = P_{0}^{sub}$ and no \textit{projection} operation is needed. In the third case, only $P_{0}^{sub-}$ is generated where $P_{0}^{sub-} = P_{0}^{sub}$. The \textit{projection} is applied though Equation \ref{equ:projection} and \ref{equ: md_update}. Overall, the first case yields two intermediate polytopes $P_1^{+}$ and $P_{1}^{'-}$. The second and third cases respectively yields $P_1 = \langle P_{0}^{sub}, M_1, d_1 \rangle$ and $P_1 = \langle P_{0}^{sub}, M'_1, d'_1 \rangle$.

This process is computing the output of one neuron, which is denoted by $\mathbb{N}(*)$ in Equation \ref{equ:layerk}. Afterward, these output polytopes will be the input to the next neuron, This process is repeated until all the neurons in the first layer are considered. Then the polytopes obtained are the output of the first layer. The computing output of one layer to its inputs is denoted by $\mathbb{L}(*)$. For a formal description, let $\mathbb{N}_{j}^{[k]}$ represent the sequence of \textit{affine transformation}, \textit{intersection and division}, and \textit{projection} for the $j$th neuron in the $k$th layer, and $\mathcal{S}_k$ be the output set of this layer with respect to an input set to this layer, then we have 
\begin{equation}
   \mathcal{S}_{k} = \mathbb{N}_{m}^{[k]}(\mathbb{N}_{m\text{-}1}^{[k]}(\dots(\mathbb{N}_1^{[k]}(\mathcal{S}_{k-1}))))
   \label{equ:layerk}
\end{equation}.
Let $\mathbb{L}_{k}$ be the process of Equation \ref{equ:layerk}, and $\mathcal{O}_k$ be the output of the $k$th layer to an input $P_0$ to the network, then we have
\begin{equation}
    \mathcal{O}_k = \mathbb{L}_{k}(\mathbb{L}_{[k\text{-}1]}(\dots (\mathbb{L}_1(P_0))))
    \label{equ:output_klayer}
\end{equation}

\subsection{Mapping Back to Input}
Let the unsafe range of output be a set of halfspaces $\mathcal{H}_{us}: \textbf{a}^{\top}_u\textbf{x} + c_u \leq 0$. Suppose we have an output polytope $P_k\text{=}\langle P_0^{sub}, M_k, d_k \rangle$. Similarly with the Equation \ref{equ:vertice_multiply}, we have 
\begin{equation*}
    \textbf{x}_k = M_k\textbf{x}_0^{sub} + d_k 
\end{equation*}
By combining $\mathcal{H}_{us}$ with this equation, we can derive that 
\begin{align}
    \textbf{a}^{\top}_u \textbf{x}_k + c_u & = \textbf{a}^{\top}_u(M_k\textbf{x}_0^{sub} + d_k) +c_u \nonumber \\
    & = (\textbf{a}^{\top}_u M_k)\textbf{x}_0^{sub} + (\textbf{a}_u^{\top}d_k + c_u) \leq 0
    \label{equ:mapH}
\end{align}
and that $\textbf{a}^{\top}_0 = \textbf{a}^{\top}_u M_k$ and $c_0 = \textbf{a}^{\top}_u d_k + c_u$ are parameters of the mapped halfspace in the $P_0^{sub}$'s space. To determine unsafety, we can check the existence of vertices that are contained in the mapped halfspaces. To extract the unsafe set in the input $P_0$, we can apply a series of the\textit{ intersection and division} operations to the $P_0^{sub}$s (linear regions) with the $\mathcal{H}_{us}$. Negative polytopes generated from such operation are the input sets that violate the safety requirement.

\begin{algorithm}[tb]
    \caption{Reachable set computation of a neural network}
    \label{al:all}
    \textbf{Input}: $n, P_0$ \quad  \comm{\# layer index, input polytope}\\
    \textbf{Output}: $\mathcal{O}_n$ \quad\comm{\# reachable set of the $n$th layer}
    \begin{spacing}{1.1}
    \begin{algorithmic}[1]
        \Procedure{$\mathcal{O}_n$ = layerOutput}{$n, P_0$}
            \State $\mathcal{I}_{k} = P_0$  \quad \comm{\# $\mathcal{I}_{k}$: input set to the $k$th layer}
            \For{$k = 1:n$} 
                \State $\mathcal{O}_{k} = empty$ \quad\comm{\# $\mathcal{O}_{k}$: output set of the $k$th layer}
                \For{(\textbf{parfor})~~$P$ in $\mathcal{I}_{k}$}
                    \State $\mathcal{S}_{k}$ = singleLayerOutput($k$, $P$) 
                    \State add $\mathcal{S}_{k}$ to $\mathcal{O}_k$
                \EndFor
                \State $\mathcal{I}_{k} = \mathcal{O}_{k}$
            \EndFor
            \State \textbf{return} $\mathcal{O}_n = \mathcal{O}_{k}$
        \EndProcedure 
        
        \Procedure{$\mathcal{S}_{k}$ = singleLayerOutput}{$k$, $P$}
            \State $\mathcal{S}_{k}$ = linearTransform($k$, $P$)
            \For{$i$ = $1:m$} \quad\comm{\# $m$: the number of neurons}
                \State $\mathcal{S}_{temp} = empty$
                \For{$P'$ in $\mathcal{S}_{k}$}
                    \State $P'^{+}$, $P'^{-}$ = intersectDivide($i$, $P'$)
                    \If{$P'^{-}$ is not none}
                        \State $P'^{-}$ = projectionHyerplane($i$, $P'^{-}$)
                    \EndIf
                    \State add  $P'^{+},P'^{-}$ to $\mathcal{S}_{temp}$
                \EndFor
                \State $\mathcal{S}_{k} = \mathcal{S}_{temp}$
            \EndFor
            \State \textbf{return} $\mathcal{S}_{k}$
        \EndProcedure
    \end{algorithmic}
    \end{spacing}
\end{algorithm}

\subsection{Parallelizable Algorithm}
Two potential parallelizable algorithms are proposed in this section based on the fact that input polytopes to one network layer are independent.  One is processing in parallel all the input polytopes to one layer until all their output is completed, and then repeating this procedure to the following layers until the final output of the neural network is achieved. This strategy is described in Algorithm \ref{al:all} where it is parallelized by replacing \textbf{for} with \textbf{parfor} in Line 5. To simplify this presentation, we assume that the output layer of neural networks is also with ReLU function . Accordingly, when $n$ is the number of layers in the neural network, the output of this algorithm is the final reachable set of the neural network. The description of functions are as following:
\begin{enumerate}
    \item[(1)] \textbf{layerOutput}() corresponds to the Equation \ref{equ:output_klayer}. The $\mathcal{S}_k$ and $\mathcal{O}_k$ corresponds to the symbols in Equation \ref{equ:layerk} and \ref{equ:output_klayer}, respectively.
    \item[(2)] \textbf{singleLayerOutput}() corresponds to the Equation \ref{equ:layerk}. 
    \item[(3)] \textbf{linearTransform}() corresponds to the Equation \ref{equ:linearTransform}. The $P$ is an element polytope, and the $\mathcal{S}_k$ is initialized with the transformed $P$.
    \item[(4)] \textbf{intersectDivide}() corresponds to the operation of\textit{ intersection and division.} by a hyperplane as shown in the (b), (c) and (d) of Figure \ref{fig:combine}. Its output are one $positive$ sub-polytope $P'^{+}$ and one $negative$ sub-polytope $P'^{-}$.
    \item[(5)] \textbf{projectionHyerplane}() corresponds to the projection of the $negative$ sub-polytopes on a hyperplane determined by the $i$th neuron, as shown in Equation \ref{equ:projection}. 
\end{enumerate}

\algblock{ParFor}{EndParFor}
\algnewcommand\algorithmicparfor{\textbf{parfor}}
\algnewcommand\algorithmicpardo{\textbf{do}}
\algnewcommand\algorithmicendparfor{\textbf{end\ parfor}}
\algrenewtext{ParFor}[1]{\algorithmicparfor\ #1\ \algorithmicpardo}
\algrenewtext{EndParFor}{\algorithmicendparfor}
\algtext*{EndParFor}

\begin{algorithm}[ht][tb]
    \caption{Alternative parallel algorithm}
    \label{al:all2}
    \textbf{Input}: $n, P_0$ \quad  \comm{\# layer index, input polytope}\\
    \textbf{Output}: $\mathcal{O}_n$ \quad\comm{\# reachable set of the $n$th layer}
    \begin{spacing}{1.1}
    \begin{algorithmic}[1]
        \Procedure{$\mathcal{O}_n$ = reachCompute}{$n, P_0$}
            \State $\mathcal{O}_j$ = layerOutput($j$, $P_0$) \ \comm{\# function in Algorithm 1}
            \ParFor{$P$ in $\mathcal{O}_j$}
                \State $\mathcal{I}_k = P$
                \For{$k = j:n$}
                    \State $\mathcal{O}_k = empty$
                    \For{$P'$ in $\mathcal{I}_k$}
                        \State $\mathcal{S}_{k}$ = singleLayerOutput($k$, $P'$) 
                        \State add $\mathcal{S}_k$ to $\mathcal{O}_k$
                    \EndFor
                    \State $\mathcal{I}_k$ = $\mathcal{O}_k$
                \EndFor
                \State add $\mathcal{O}_k$ to $\mathcal{O}_n$
            \EndParFor
            \State \textbf{return} $\mathcal{O}_n$
        \EndProcedure 
    \end{algorithmic}
    \end{spacing}
\end{algorithm}

However, when programmed in Python using its embedded \textbf{multiprocessing} package, this parallel algorithm will suffer from a computational burden due to copying the parent process memory in the computation of each layer. It means that when the function \textbf{singleLayerOutput()} is invoked in parallel, the memory in the parent process will be copied to those multiple created children processes. This memory copying occurs in each network layer and leads to an extra-large computational burden. This phenomenon is significant when there is a large number of polytopes in $\mathcal{I}_k$.

Therefore, an alternative method is inherited from the first one to handle that situation. Instead of processing polytopes in parallel in each layer, we choose to start processing them in parallel from a specific layer. The details are illustrated in Algorithm \ref{al:all2}. It first computes the output polytopes of the $j$ layer to the input $P_0$ by invoking the function \textbf{layerOutput} in Algorithm \ref{al:all}. Then, the algorithm from Line 3 to the end will be executed. In this process, the memory copying only happens once in Line 3 with a relatively smaller number of polytopes in $\mathcal{O}_j$. Therefore, it outperforms Algorithm \ref{al:all}.  For the complexity, the maximum number of output sets that both of algorithms can compute for a network with total $m$ hidden neurons and one input set is $2^m$. In practice, the actual number is much smaller since the \textit{intersection and division} doesn't really split intermediates polytopes in every neuron.

\section{Evaluation}

\subsection{Safety Verification of ACAS Xu Networks}
In this section, we evaluate our method against the Reluplex, Marabou, and NNV. They are currently well-known approaches that also do the exact reachability analysis of feed-forward neural networks with ReLU activation functions. The evaluation task is to verify the safety of ACAS Xu networks in \cite{julian2016policy}. Given an input domain and an unsafe domain, Reluplex and Marabou solve the satisfiability problem, which is they only determine if there are outputs satisfying unsafe conditions. While our method and NNV are computing all the exact reachable sets of the output and then determining their intersection with the unsafe area. 

The ACAS Xu networks are used to approximate a large lookup table that converts sensor measurements into maneuver advisories in an Airborne Collision Avoidance System so that the massive memory occupation by the table and the lookup time can be reduced. The set of networks contain $45$ fully-connected DNNs for different combinations of discretized parameters. Each one has five inputs, five outputs, and six hidden layers. Each layer consists of fifty neurons with ReLU activation functions. Therefore there are $300$ hidden neurons total in each network. The five inputs consist of the sensor measurements: (1) $\rho$: Distance from ownship to intruder; (2) $\theta$: Angle to intruder relative to ownship heading direction; (3) $\psi$: Heading angle of intruder relative to ownship heading direction; (4) $v_{own}$: Speed of ownship; (5) $v_{int}$: Speed of intruder. The outputs are advisory scores for five different actions, where the lowest score corresponds to the best action. The five actions are respectively, clear of conflict (COC), weak right, strong right, weak left, and strong left. There are ten safety properties $\phi_1,\dots,\phi_{10}$ which specify input bounds and linear constraints on the output. The details can be found in the Appendix of \cite{katz2017reluplex}.

In the experiment, All the 45 networks are tested on Property 1,2,3 and 4. The hardware configuration is Intel
Core i7-6700 CPU @3.4GHz$\times$8 Processor, 64 GB Memory, 64-bit Ubuntu 18.04\urlstyle{sf}\footnote{Codes are available online at \urlstyle{rm}\url{https://github.com/verivital/FaceLattice}}. The summary of results are shown in Table \ref{tab:summary}. The details are shown in Table \ref{tab:p1},\ref{tab:p2},\ref{tab:p3} and \ref{tab:p4} in Appendix. In the implementation, Marabou DNC, NNV, ReluVal and our method that support parallel computation are all assigned with 8 processors. While Reluplex and Marabou is implemented in a nonparallel way. The following is the evaluation of each method:
\begin{itemize}
    \item{\textbf{Reluplex}:} As the summary indicates, Reluplex exhibits a relatively lower efficiency compared to other methods. Besides, an incorrect result is found in our implementation. The test on the $network_{17}$ and  Property 2 identifies an adversarial input that generates an output $[-0.00, 0.30, 0.18, 0.19, 0.20]$ . While the desired output for Property 2 is that the score for COC (the first element of output) is not maximum. Clearly, this output doesn't violate this condition. 
    \item{\textbf{Marabou}:} For the performance of Marabou, It shows an improvement in the efficiency compared to Reluplex. But It still takes a significant time to complete tasks. Its scalability to larger and deeper neural networks is unclear.
    \item{\textbf{Marabou DnC}:} It is a divide-and-conquer solving mode of the Marabou, where the input range is partitioned into sub-ranges. Compared to Reluplex and Marabou, it shows a higher efficiency on Property 1 and 2. However, it fails to outperform them on Property 3 and 4. This performance inconsistency is related to its configuration. Its configuration consists of parameters such as initial-divides and online divides flags which are the number of times to bisect the input region. Such parameters can greatly affect performance. Besides, their optimal selection differs between different networks. Therefore, their optimal performance can't be always achieved. The parameters we applied for this implementation are the default values.
    \item{\textbf{ReluVal}:} Compared to other tools, it exhibits a higher efficiency in some cases such as Property 1, but it can be very slow in other cases such as Property 2. Its performance is also inconsistent. ReluVal is constructed with two modes, CHECK\_ADV\_\\MODE and CHECK\_NORM\_MODE. The first mode can quickly breathe search for counter-examples but it can't guarantee to identify them if they truly exist. While the second mode can return adversarial examples if exist but its high efficiency can't be guaranteed. As recommended by their authors, The first mode is applied in the beginning. For those tasks without finding adversarial examples, the second mode is then applied. But the performance of the second mode can be greatly affected by its parameters in the configuration. For instance, the parameter "depth" indicates from which depth the algorithm should start checking adversarial examples. Its optimal value also differs between networks. Thus, its optimal performance can't be guaranteed. The parameters we utilized are default values for this experiment. 
    \item{\textbf{NNV Exact Star}:} It shows high efficiency in Property 3 and 4 but fails to complete most of the tasks in Property 1 and 2. As introduced in the beginning, this inconsistency is probably related to the optimization processes involved in the computation of output sets.
    \item{\textbf{Face Lattice}:} Our method shows a relatively higher and more consistent efficiency. In Property 1, our method fails to outperform ReluVal because we compute the exact output sets first and then solve the satisfiability problem while the large input range of Property 1 yields a significant number of linear regions to compute. However, the output sets obtained from the test on Property 1 can be reused for Property 2 where networks have the same output reachable sets.
\end{itemize}

\begin{table}[h]
\centering
\renewcommand{\arraystretch}{1.3}
\renewcommand{\tabcolsep}{1.3mm}
\begin{adjustbox}{angle=0, max width=\textwidth}
\begin{tabular}{l|cccc|cccc}
\hline
\multirow{2}{*}{\textbf{Methods}} &   \multicolumn{4}{c}{\textbf{ACAS XU $\phi_1$}} & \multicolumn{4}{c}{\textbf{ACAS XU $\phi_2$}} \\ \cline{2-9}
 {} & SAT & UNSAT    & TIMEOUT      & TIME &SAT &UNSAT  &TIMEOUT      & TIME  \\  \cline{1-1}
Reluplex          & 0                   & 28         & 17                 & $>462208$        & \red{40}                   & \red{3}                      & 2                 & $>153292$                      \\ 
Marabou           & 0                   & 38         & 7                 & $>335933$         & 39                   & 3                      & 3                 & $>160024$                   \\ 
Marabou DnC       & 0                   & 45         & 0                 & 99754         & 39                   & 4                      & 2                 & $>89294$                  \\ 
ReluVal           & 0                   & 45         & 0                 & 467          & 39                   & 6                      & 0                & 28019                    \\ 
NNV Exact Star    & 0                   & 14         & 31                 & $>711271$        & 10                   & 4                      & 31                & $>370$                    \\ 
Our Method   & 0                   & 45         & 0                 & 31312            & 39                   & 6                      & 0                & 255                 \\ \hline 
\end{tabular}
\end{adjustbox}
\end{table}
\begin{table}[h]
\centering
\renewcommand{\arraystretch}{1.3}
\renewcommand{\tabcolsep}{1.3mm}
\begin{adjustbox}{angle=0, max width=\textwidth}
\begin{tabular}{l|cccc|cccc}
\hline
\multirow{2}{*}{\textbf{Methods}} &   \multicolumn{4}{c}{\textbf{ACAS XU $\phi_3$}} & \multicolumn{4}{c}{\textbf{ACAS XU $\phi_4$}} \\ \cline{2-9}
 {} & SAT & UNSAT    & TIMEOUT      & TIME &SAT &UNSAT  &TIMEOUT      & TIME  \\  \cline{1-1}
Reluplex          & 3                   & 42         & 0                 & 28454        & 3                   & 42                      & 0                 & 11880                      \\ 
Marabou           & 3                   & 42         & 0                 & 19466         & 3                   & 42                      & 0                 & 8470                   \\ 
Marabou DnC       & 3                   & 41         & 1                 & $>56322$         & 3                   & 42                      & 0                 & 25110                  \\ 
ReluVal           & 3                   & 42         & 0                 & 759          & 3                   & 42                      & 0                & 56                    \\ 
NNV Exact Star    & 3                   & 42         & 0                 & 7457        & 3                   & 42                      & 0                & 1157                    \\ 
Our Method   & 3                   & 42         & 0                 & 511            & 3                   & 42                      & 0                & 296                 \\ \hline 
\end{tabular}
\end{adjustbox}
\caption{Summary of the performance of each method on Property 1,2,3 and 4. The SAT means there exists output that satisfies the unsafe condition, which indicates that the neural network is unsafe. While the UNSAT means that there doesn't exist such output, which indicates the network is safe. The TIMEOUT is set to 5h for all the tasks. The TIME represents the total running time (seconds) of one method on 45 networks.}
\label{tab:summary}
\end{table}
\begin{figure}[ht]
    \centering
    \includegraphics[scale = 0.45]{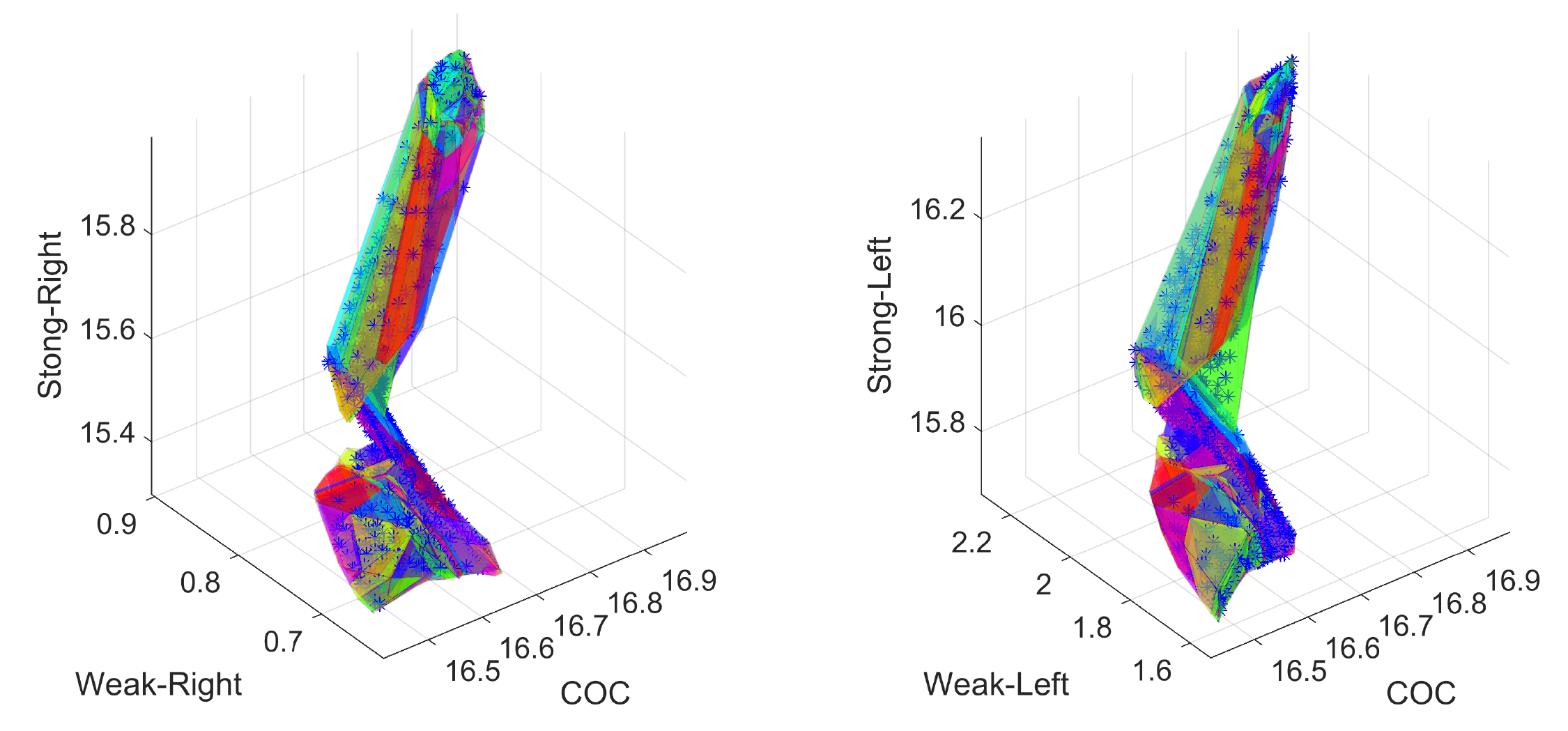}
    \caption{Exact reachable set of the network $N_{4\_7}$ on Property 4. The blue star points are the output w.r.t. 2000 random input samples. They all locate on or inside the output sets.}
    \label{fig:N37}
\end{figure}
\begin{figure}[ht]
    \centering
    \includegraphics[scale = 0.45]{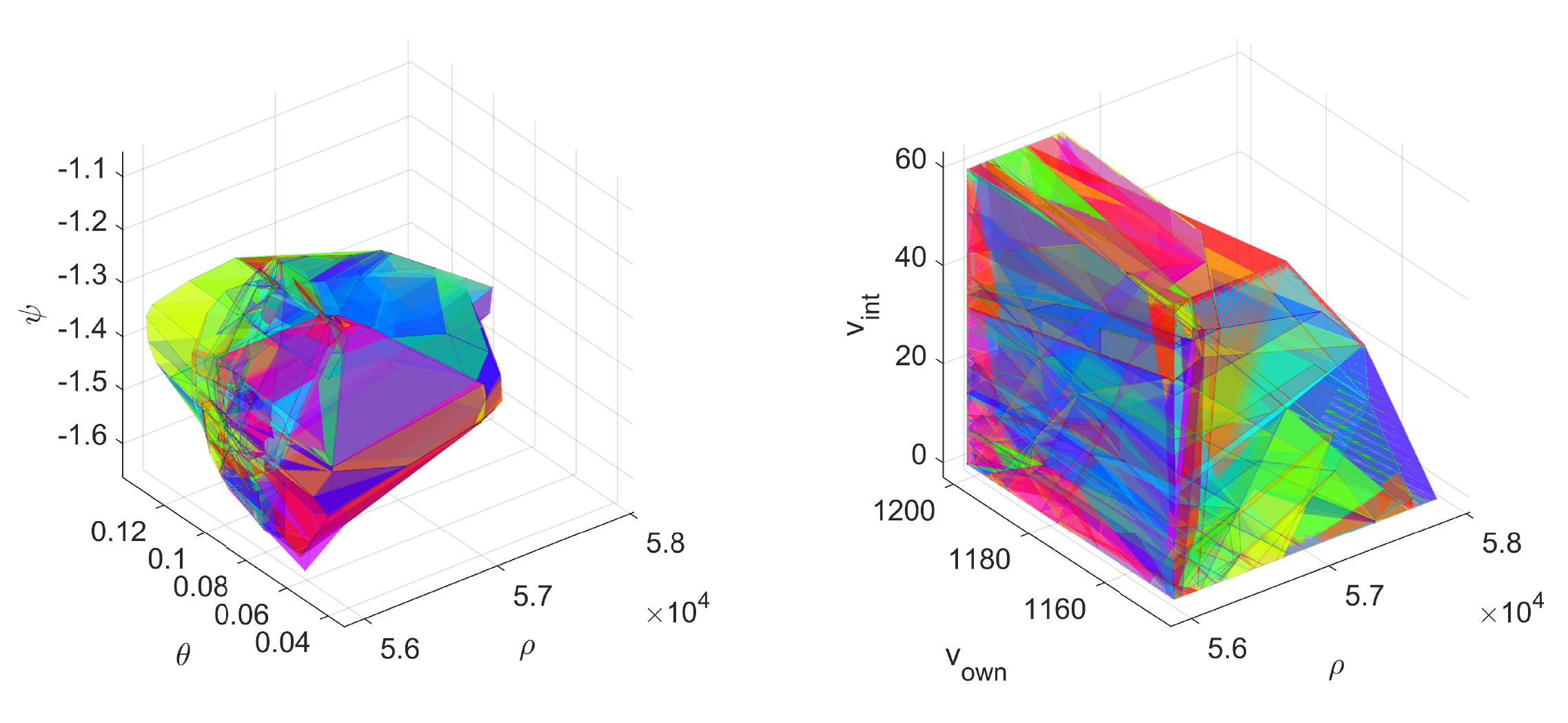}
    \caption{Complete input sets that lead to the property $\phi_{2}$ violation on the network $N_{1\_2}$. It is a union of 104 polytopes.}
    \label{fig:3_7_track}
\end{figure}

In addition to the high efficiency, our approach can also visualize the exact reachability of a neural network so that an intuitive inspection can first be conducted for safety properties. An example of the exact reachable set of the network $N_{4\_7}$ is demonstrated in Figure \ref{fig:N37}, whose lower bound and upper bound of the input are $[1500, \text{-}0.06, 3.1, 1000,\\ 700]$ and $[1800, 0.06, 3.14, 1200, 800]$. Additionally, 2000 random inputs are sampled to demonstrate the method's correctness. The reachable sets are visualized separately in two figures. With such output sets, we can verify different safety properties, instead of running the algorithm repeatedly for each property as Reluplex and Marabou do. Besides, our method can also compute the exact input set that leads to property violations. For instance, the inputs that make the network $N_{1\_2}$ violate Property2 are computed as shown in Figure \ref{fig:3_7_track}. It is a union of 104 polytopes. The element belonging to these polytopes will yield property violations. For the element that doesn't belong to them but are contained in the input range, they won't cause any violation.

\subsection{Reachability Analysis of Microbenchmarks}
To further evaluate our approach, we compare them on a set of microbenchmarks that are proposed in work \cite{dutta2018output}. These benchmarks consisting of neural networks are created from two different sources. The networks $N_2\text{-}N_4$ are trained with some analytical functions. While the networks $N_{11}\text{-}N_{14}$ come from the unwindings of a closed-loop controller and a plant model. Here we check these networks' safety on some synthetic unsafe domains. Only Marabou DnC and NNV are considered because they support parallel computation. ReluVal is not included as we are not able to apply ReluVal to these new networks. Problems have been reported to the author but haven't been solved. For the networks $N_2\text{-}N_4$ which have only one hidden layer, we equally partition the input range for our method and NNV into 1024 subsets, so that the parallel computation can be applied. The details of each benchmark and the performance of each approach are shown in Table \ref{tab:2}. Most of the networks consist of thousands of neurons. Dealing with such a large size of networks, our approach also exhibits a high efficiency. While other methods fail to complete all the tasks.
\begin{table}
\footnotesize
\renewcommand{\arraystretch}{1.3}
\renewcommand{\tabcolsep}{1.3mm}
\centering
\begin{tabular}{c|ccc|cc|c|cc}
\hline
\multirow{2}{*}{\textbf{ID}} & \multirow{2}{*}  & \multirow{2}{*}  & \multirow{2}{*} & \multicolumn{2}{c|}{\textbf{Our Method}} & \textbf{Marabou DnC} &\multicolumn{2}{c}{\textbf{NNV}}  \\ \cline{2-9} 
                    & {$x$}                      & {$k$}                      & {$m$}                     & $N_p$            & $time$           & $time$  & $N_p$  & $time$ \\ \cline{1-9} 
$N_{2}$             & 2                     & 1                     & 500                  & 16760             & 25             & 36        & 16760  &635      \\ 
$N_{3}$             & 2                     & 1                     & 500                  & 4339             & 8              & 15         & 4339  & 53      \\ 
$N_{4}$             & 2                     & 1                     & 1000                 & 192709           & 803           & 157        &$\times$ & TIMEOUT      \\ 
$N_{11}$            & 3                     & 9                     & 1427                 & 25312            & 449            & TIMEOUT      & 25312     & 1824     \\ 
$N_{12}$            & 3                     & 14                    & 2292                 & 25312            & 837            & TIMEOUT       & 25312    & 1690 \\ 
$N_{13}$            & 3                     & 19                    & 3057                 & 462833            & 20033           & TIMEOUT     &$\times$   & TIMEOUT      \\ 
$N_{14}$            & 3                     & 24                    & 3822                & 236062            & 13991           & TIMEOUT      &$\times$  & TIMEOUT      \\ \hline
\end{tabular}
\caption{Performance results on microbenchmarks. The label $x$, $k$ and $m$ respectively denote the number of input, the number of layers and total number of ReLU neurons. $N_p$ is the number of output reachable sets and $time$($second$) is the running time. The TIMEOUT is set to 10 hours.}
\label{tab:2}
\end{table}

\section{Conclusion and Future Work}
 The exact reachability analysis of neural networks is a hard problem, but it is essential to conducting robustness verification.  In this paper, we presented a polytope-based and parallelizable method to conduct exact reachability analysis of the feed-forward neural networks with ReLU activation. Instead of using the $H$ and $V$ representations, we construct polytopes with the face lattice, a structure encoding complete combinatorial information. Through comparing to state-of-the-art methods like Reluplex, Marabou, NNV, and ReluVal, we can verify its correctness and high efficiency. As discussed, the exact reachable set can be reusable for different properties. The comparison includes the verification of a set of real-world ACAS Xu neural networks and the reachability computation of some other benchmarks. To improve the efficiency and convenience, we can also partition a wide-range input into smaller subsets and build a library for their reachable sets so that further analysis of any input contained in such range can be reduced to merely looking up the library. Overall, there are multiple directions to improve our approach. In the future, we plan to explore more concise and efficient polytope representations than the face lattice because it becomes complicated for high dimensional polytopes. Another potential enhancement is to detect the polytopes in the input set that determine the output range so that the redundant polytopes can be ignored and the computation burden can be reduced. Additionally, we also consider extending our approach to convolutional neural networks (CNNs) in the future.

\newpage

\bibliographystyle{splncs04}
\bibliography{nfm.bib}

\appendix

\begin{table}
\centering
\renewcommand{\arraystretch}{1.5}
\renewcommand{\tabcolsep}{1.5mm}
\begin{adjustbox}{angle=0, max width=\textwidth}
\begin{tabular}{c|cccccccccccccc}
\hline
                              & \multicolumn{2}{c|}{\textbf{Reluplex}} & \multicolumn{2}{c|}{\textbf{Marabou}} & \multicolumn{2}{c|}{\textbf{Marabou DnC}} & \multicolumn{2}{c|}{\textbf{ReluVal}} & \multicolumn{3}{c|}{\textbf{NNV Exact Star}} & \multicolumn{3}{c}{\textbf{Our Method}} \\ \cline{2-15} 
\multirow{-2}{*}{\textbf{ID}} &  $VT$  & \multicolumn{1}{c|}{$V$}    & $VT$                          & \multicolumn{1}{c|}{$V$}   & $VT$                            & \multicolumn{1}{c|}{$V$}   & $VT$                            & \multicolumn{1}{c|}{$V$}  & $N_p$ &  $VT$ & \multicolumn{1}{c|}{$V$}   & $N_p$  &  $VT$     & \multicolumn{1}{c}{$V$}       \\ \hline
$N_{11}$                      & \cellcolor[HTML]{C7F7C7}1875  & UNSAT  & \cellcolor[HTML]{C7F7C7}393  & UNSAT & \cellcolor[HTML]{C7F7C7}74   & UNSAT  &\cellcolor[HTML]{C7F7C7}0.17    & UNSAT   & \cellcolor[HTML]{FFFFC7}39810 & \cellcolor[HTML]{C7F7C7}4203  & UNSAT & \cellcolor[HTML]{FFFFC7}39810 & \cellcolor[HTML]{C7F7C7}81.70     & UNSAT       \\ 
$N_{12}$                      & \cellcolor[HTML]{C7F7C7}2697  & UNSAT  & \cellcolor[HTML]{C7F7C7}802  & UNSAT & \cellcolor[HTML]{C7F7C7}129   & UNSAT  &\cellcolor[HTML]{C7F7C7}0.31    & UNSAT   & \cellcolor[HTML]{FFFFC7}45606  & \cellcolor[HTML]{C7F7C7}4912 & UNSAT & \cellcolor[HTML]{FFFFC7}45606 & \cellcolor[HTML]{C7F7C7}92.22     & UNSAT       \\ 
$N_{13}$                      & \cellcolor[HTML]{C7F7C7}$>18000$  & $\times$  & \cellcolor[HTML]{C7F7C7}2319  & UNSAT & \cellcolor[HTML]{C7F7C7}1040   & UNSAT  &\cellcolor[HTML]{C7F7C7}2.92    & UNSAT   & \cellcolor[HTML]{FFFFC7}114219 & \cellcolor[HTML]{C7F7C7}12665  & UNSAT & \cellcolor[HTML]{FFFFC7}114219 & \cellcolor[HTML]{C7F7C7}247.14     & UNSAT       \\
\hline
\end{tabular}
\end{adjustbox}
\caption{Property 1}
\label{tab:p1}
\end{table}

\begin{table}
\centering
\renewcommand{\arraystretch}{1.5}
\renewcommand{\tabcolsep}{1.5mm}
\begin{adjustbox}{angle=0, max width=\textwidth}
\begin{tabular}{c|cccccccccccccc}
\hline
                              & \multicolumn{2}{c|}{\textbf{Reluplex}} & \multicolumn{2}{c|}{\textbf{Marabou}} & \multicolumn{2}{c|}{\textbf{Marabou DnC}} & \multicolumn{2}{c|}{\textbf{ReluVal}} & \multicolumn{3}{c|}{\textbf{NNV Exact Star}} & \multicolumn{3}{c}{\textbf{Our Method}} \\ \cline{2-15} 
\multirow{-2}{*}{\textbf{ID}} &  $VT$  & \multicolumn{1}{c|}{$V$}    & $VT$                          & \multicolumn{1}{c|}{$V$}   & $VT$                            & \multicolumn{1}{c|}{$V$}   & $VT$                            & \multicolumn{1}{c|}{$V$}  & $N_p$ &  $VT$ & \multicolumn{1}{c|}{$V$}   & $N_p$  &  $VT$     & \multicolumn{1}{c}{$V$}       \\ \hline
$N_{11}$                      & \cellcolor[HTML]{C7F7C7}4888  & UNSAT  & \cellcolor[HTML]{C7F7C7}4290  & UNSAT & \cellcolor[HTML]{C7F7C7}1887   & UNSAT  &\cellcolor[HTML]{C7F7C7}1.04    & UNSAT   & \cellcolor[HTML]{FFFFC7}39810 & \cellcolor[HTML]{C7F7C7}9  & UNSAT & \cellcolor[HTML]{FFFFC7}39810 & \cellcolor[HTML]{C7F7C7}0.67     & UNSAT       \\ 
$N_{12}$                      & \cellcolor[HTML]{C7F7C7}396  & SAT  & \cellcolor[HTML]{C7F7C7}760   & SAT & \cellcolor[HTML]{C7F7C7}450   & SAT  &\cellcolor[HTML]{C7F7C7}2.84   & SAT   & \cellcolor[HTML]{FFFFC7}45606 & \cellcolor[HTML]{C7F7C7}11  & SAT & \cellcolor[HTML]{FFFFC7}45606 & \cellcolor[HTML]{C7F7C7}0.64     & SAT       \\ 
$N_{13}$                      & \cellcolor[HTML]{C7F7C7}1704  & SAT  & \cellcolor[HTML]{C7F7C7}1122  & SAT & \cellcolor[HTML]{C7F7C7}9465   & SAT  &\cellcolor[HTML]{C7F7C7}5.55   & SAT   & \cellcolor[HTML]{FFFFC7}114219 & \cellcolor[HTML]{C7F7C7}28  & SAT & \cellcolor[HTML]{FFFFC7}114219 & \cellcolor[HTML]{C7F7C7}1.82     & SAT       \\
\hline
\end{tabular}
\end{adjustbox}
\caption{Property 2}
\label{tab:p2}
\end{table}

\begin{table}
\centering
\renewcommand{\arraystretch}{1.5}
\renewcommand{\tabcolsep}{1.5mm}
\begin{adjustbox}{angle=0, max width=\textwidth}
\begin{tabular}{c|cccccccccccccc}
\hline
                              & \multicolumn{2}{c|}{\textbf{Reluplex}} & \multicolumn{2}{c|}{\textbf{Marabou}} & \multicolumn{2}{c|}{\textbf{Marabou DnC}} & \multicolumn{2}{c|}{\textbf{ReluVal}} & \multicolumn{3}{c|}{\textbf{NNV Exact Star}} & \multicolumn{3}{c}{\textbf{Our Method}} \\ \cline{2-15} 
\multirow{-2}{*}{\textbf{ID}} &  $VT$  & \multicolumn{1}{c|}{$V$}    & $VT$                          & \multicolumn{1}{c|}{$V$}   & $VT$                            & \multicolumn{1}{c|}{$V$}   & $VT$                            & \multicolumn{1}{c|}{$V$}  & $N_p$ &  $VT$ & \multicolumn{1}{c|}{$V$}   & $N_p$  &  $VT$     & \multicolumn{1}{c}{$V$}       \\ \hline
$N_{11}$                      & \cellcolor[HTML]{C7F7C7}6156  & UNSAT  & \cellcolor[HTML]{C7F7C7}3637  & UNSAT & \cellcolor[HTML]{C7F7C7}$>18000$   & $\times$  &\cellcolor[HTML]{C7F7C7}176.89     & UNSAT   & \cellcolor[HTML]{FFFFC7}70631 & \cellcolor[HTML]{C7F7C7}1463  & UNSAT & \cellcolor[HTML]{FFFFC7}70631 & \cellcolor[HTML]{C7F7C7}110.89     & UNSAT       \\ 
$N_{12}$                      & \cellcolor[HTML]{C7F7C7}4942  & UNSAT  & \cellcolor[HTML]{C7F7C7}3418  & UNSAT & \cellcolor[HTML]{C7F7C7}3221    & UNSAT  &\cellcolor[HTML]{C7F7C7}4.30     & UNSAT    & \cellcolor[HTML]{FFFFC7}39529 & \cellcolor[HTML]{C7F7C7}1026 & UNSAT & \cellcolor[HTML]{FFFFC7}39529  & \cellcolor[HTML]{C7F7C7}65.63    & UNSAT       \\
$N_{13}$                      & \cellcolor[HTML]{C7F7C7}1134  & UNSAT  & \cellcolor[HTML]{C7F7C7}1795  & UNSAT & \cellcolor[HTML]{C7F7C7}10858   & UNSAT  &\cellcolor[HTML]{C7F7C7}6.91     & UNSAT    & \cellcolor[HTML]{FFFFC7}12274 & \cellcolor[HTML]{C7F7C7}299  & UNSAT & \cellcolor[HTML]{FFFFC7}12274  & \cellcolor[HTML]{C7F7C7}20.17     & UNSAT       \\
\hline
\end{tabular}
\end{adjustbox}
\caption{Property 3}
\label{tab:p3}
\end{table}

\begin{table}
\centering
\renewcommand{\arraystretch}{1.5}
\renewcommand{\tabcolsep}{1.5mm}
\begin{adjustbox}{angle=0, max width=\textwidth}
\begin{tabular}{c|cccccccccccccc}
\hline
                                & \multicolumn{2}{c|}{\textbf{Reluplex}}                      & \multicolumn{2}{c|}{\textbf{Marabou}} & \multicolumn{2}{c|}{\textbf{Marabou DnC}} & \multicolumn{2}{c|}{ \textbf{ReluVal}} & \multicolumn{3}{c|}{\textbf{NNV Exact Star}} & \multicolumn{3}{c}{\textbf{Our Method}} \\ \cline{2-15} 
\multirow{-2}{*}{\textbf{ID}}  &  $VT$  & \multicolumn{1}{c|}{$V$}    & $VT$                          & \multicolumn{1}{c|}{$V$}   & $VT$                            & \multicolumn{1}{c|}{$V$}   & $VT$                            & \multicolumn{1}{c|}{$V$}  & $N_p$ &  $VT$ & \multicolumn{1}{c|}{$V$}   & $N_p$ &  $VT$ & \multicolumn{1}{c}{$V$}       \\ \hline
$N_{11}$                                              & \cellcolor[HTML]{C7F7C7}{\color[HTML]{000000} 1291} & UNSAT & \cellcolor[HTML]{C7F7C7}1929  & UNSAT & \cellcolor[HTML]{C7F7C7}11379                                     & UNSAT               & \cellcolor[HTML]{C7F7C7}1.90                                     & UNSAT             & \cellcolor[HTML]{FFFFC7}19143  & \cellcolor[HTML]{C7F7C7}398  & UNSAT & \cellcolor[HTML]{FFFFC7}19143  & \cellcolor[HTML]{C7F7C7}33.09 & UNSAT       \\ 
$N_{12}$                                              & \cellcolor[HTML]{C7F7C7}{\color[HTML]{000000} 1267} & UNSAT & \cellcolor[HTML]{C7F7C7}2012  & UNSAT & \cellcolor[HTML]{C7F7C7}8629                                       & UNSAT               & \cellcolor[HTML]{C7F7C7}2.42                                      & UNSAT             & \cellcolor[HTML]{FFFFC7}13143   & \cellcolor[HTML]{C7F7C7}258  & UNSAT & \cellcolor[HTML]{FFFFC7}13143  & \cellcolor[HTML]{C7F7C7}23.16 & UNSAT       \\
\hline
\end{tabular}
\end{adjustbox}
\caption{Property 4}
\label{tab:p4}
\end{table}

\end{document}